%% file: polaradamw.tex
\documentclass{article}

\usepackage[margin=1in]{geometry}
\usepackage{hyperref}
\hypersetup{
  colorlinks=true,
  linkcolor=NavyBlue,
  citecolor=NavyBlue,
  urlcolor=NavyBlue,
}
\usepackage{mmll-2}
\crefname{section}{Section}{Sections}
\Crefname{section}{Section}{Sections}
\crefname{appendix}{Appendix}{Appendices}
\Crefname{appendix}{Appendix}{Appendices}
\crefname{figure}{Figure}{Figures}
\Crefname{figure}{Figure}{Figures}
\crefname{table}{Table}{Tables}
\Crefname{table}{Table}{Tables}
\crefname{theorem}{Theorem}{Theorems}
\Crefname{theorem}{Theorem}{Theorems}
\crefname{lemma}{Lemma}{Lemmas}
\Crefname{lemma}{Lemma}{Lemmas}
\crefname{proposition}{Proposition}{Propositions}
\Crefname{proposition}{Proposition}{Propositions}
\crefname{remark}{Remark}{Remarks}
\Crefname{remark}{Remark}{Remarks}
\crefname{equation}{Eq.}{Eqs.}
\Crefname{equation}{Eq.}{Eqs.}

\newcommand{\polar}{\operatorname{polar}}
\newcommand{\Hom}{\operatorname{Hom}}
\newcommand{\Frob}{\mathrm{F}}
\newcommand{\R}{\mathbb{R}}

\title{PolarAdamW: Disentangling Spectral Control and Schur Gauge-Equivariance
       in Matrix Optimisation}

\author{%
  Haozhou Zhang \\
  Department of Mathematics and Statistics \\
  University of Idaho \\
  \texttt{zhan3063@vandals.uidaho.edu}
}

\begin{document}
\maketitle

\begin{abstract}
Muon's matrix-level update couples two distinct effects: spectral
control via a polar map, and equivariance under orthogonal changes of
multiplicity-space basis (Schur gauge-equivariance). We separate them
with PolarAdamW, a controlled hybrid that preserves Muon's polar
spectral-norm control but breaks the gauge-equivariance, since AdamW's
coordinatewise preconditioner is basis-dependent. Algorithmically,
PolarAdamW applies Muon's Newton--Schulz polar map to AdamW's
preconditioned direction rather than to raw momentum, at per-iteration
wall-time comparable to Muon. We prove that Muon's polar step is Schur
gauge-equivariant on multiplicity matrices while AdamW's coordinatewise
step is not.

On DeiT-Tiny trained from scratch on four independently sampled
100-class subsets of ImageNet-1k, where multiplicity-basis freedom is
trivial, PolarAdamW outperforms Muon by $+1.93$ pp in test accuracy on
average and AdamW by $+9.5$ pp; under the 300-epoch DeiT-style recipe,
it remains ahead of Muon by $+1.37$ pp and AdamW by $+5.80$ pp on
average. On SO(3)-equivariant 3D point-cloud regression, where
multiplicity-basis freedom is non-trivial, the ordering reverses: Muon
outperforms PolarAdamW at every audited capacity, and the gap widens
with capacity. Both matrix-polar optimisers continue to outperform
AdamW. This double dissociation separates spectral control from Schur
gauge-equivariance: the first composes well with AdamW preconditioning
on standard transformers, while the second becomes consequential when
multiplicity-basis freedom is structurally non-trivial.
\end{abstract}

\section{Introduction}
\label{sec:setting}

While adaptive optimisers like
AdamW~\citep{loshchilov2019decoupledweightdecayregularization}
precondition each parameter coordinate by its second-moment estimate, matrix-structured
polar optimisers like Muon~\citep{Jordan2024} apply a polar map to each
update matrix, controlling its spectral norm. Both are strong alternatives
to stochastic gradient descent (SGD) with momentum on transformer
training, but have been developed as separate lines of research.

Whether the two mechanisms (per-coordinate adaptation and matrix-level
spectral control) compete or compose remains unsettled. We introduce
\textbf{PolarAdamW} as a controlled artefact: it applies Muon's polar
map to AdamW's preconditioned direction rather than to raw momentum,
at per-iteration wall-time comparable to Muon. Muon polarises a
momentum matrix directly; PolarAdamW first forms the
AdamW-preconditioned direction, whose entries have already been
rescaled by coordinatewise second-moment statistics, and then applies
the same polar map. The two optimisers share matrix-level spectral
control but differ in whether coordinatewise adaptive preconditioning
is allowed to shape the matrix direction before polarisation.

This contrast sets up a double dissociation (a pair of regimes where
each mechanism, taken alone, predicts the opposite ordering): if
coordinatewise preconditioning is the useful piece, PolarAdamW should
outperform Muon on ordinary transformer matrices, where
multiplicity-space basis freedom is trivial; if basis-equivariance is
the useful piece, the PolarAdamW advantage should weaken or reverse
on equivariant architectures, where this basis freedom is non-trivial.

The first prediction holds. Training
DeiT-Tiny~\citep{touvron2021trainingdataefficientimagetransformers}
from scratch on four independently sampled 100-class subsets of
ImageNet-1k, PolarAdamW outperforms Muon by $+1.93$ pp test accuracy
on average (4-seed paired, all four differences positive) and AdamW
by $+9.5$ pp. Under the 300-epoch DeiT-style recipe, the lead is preserved across
the same four seeds:
PolarAdamW remains ahead of Muon by $+1.37$ pp on average and ahead
of AdamW by $+5.80$ pp, with all four PolarAdamW $-$ Muon paired
differences positive.

The second prediction also holds. On an SO(3)-equivariant 3D
point-cloud regression testbed, in a paired-seed audit across
capacities $h_c \in \{16, 32, 64, 128\}$, Muon outperforms PolarAdamW
at every audited capacity, and the gap widens monotonically with
capacity. Both matrix-polar optimisers continue to outperform AdamW.
This double dissociation supports the regime-wise separation:
coordinatewise preconditioning is beneficial when the gauge is trivial,
while basis-equivariance becomes consequential when it is not.

We map out the rest of the paper as follows.
\cref{sec:related} relates PolarAdamW to adaptive
diagonal methods, post-Muon variants, and equivariant network design.
\cref{sec:polaradamw} introduces the optimiser.
\cref{sec:experiments-transformer} reports the DeiT-Tiny
experiments.
\cref{sec:schur-setting,sec:gauge-equivariance} develop the Schur multiplicity
formalism and prove the main equivariance results:
\cref{thm:basis-equiv} for Muon's polar step and
\cref{prop:adamw-not-equiv} for AdamW's coordinatewise step.
\cref{sec:experiments-equivariant} reports the SO(3)
experiments. \cref{sec:open} discusses scope and limitations.

\section{Related and Concurrent Work}
\label{sec:related}

PolarAdamW sits between two optimiser threads that have largely been studied
separately: adaptive coordinatewise preconditioning (the AdamW side) and
matrix-structured polar optimisers (the Muon side). We also review
equivariant network design, where multiplicity-space basis freedom makes the
distinction between these optimiser mechanisms concrete. We address each in
turn, marking the basis-dependence properties that make their interaction
non-trivial.

\subsection{Adaptive diagonal methods}
\label{sec:related-adamw}

Coordinatewise adaptive preconditioning is the dominant lever in modern
training: Adam~\citep{kingma2015adam} divides
the update by a per-coordinate second-moment estimate, and
AdamW~\citep{loshchilov2019decoupledweightdecayregularization} decouples
weight decay from the Adam preconditioner. Recent analyses note that
this preconditioner is basis-dependent: the same training problem in
different parameter bases can produce measurably different
trajectories~\citep{liu2025adamgaussnewtoncomparativestudy,zhang2025understandingadamrequiresbetter}.
In ordinary architectures the basis is fixed by the parameterisation
and the basis-dependence is invisible; in Schur multiplicity spaces it
becomes a structural property of the optimiser rather than of the
architecture.

\subsection{Matrix-structured polar optimisers}
\label{sec:related-polar}

Muon~\citep{Jordan2024} replaces the per-coordinate rescale with a polar
map applied to the matrix-valued momentum, computed via a Newton--Schulz
iteration. Two analyses motivate the polar step:
\citet{Bernstein2025Deriving} derives it from an output-perturbation
argument on a single linear layer, with the RMS-to-RMS operator norm
of $\Delta W$ controlling the induced output change; the Modula
documentation~\citep{Modula} develops the spectral-norm/Stiefel-manifold
view of Muon-style updates and their retractions. Polar
Express~\citep{amsel2025polarexpressoptimalmatrix} replaces the fixed
Newton--Schulz coefficients with minimax-optimised polynomial updates
for the polar decomposition.
NS-RGS~\citep{peng2026nsrgsnewtonschulzbasedriemannian} uses
Newton--Schulz as an approximate polar projector inside a Riemannian
gradient scheme for orthogonal-group synchronisation.
Shampoo~\citep{gupta2018shampoopreconditionedstochastictensor} is an
earlier matrix preconditioner with a different mechanism:
Kronecker-factored accumulated second moments, rather than a per-step
polar map of the momentum.
SOAP~\citep{vyas2025soap} combines Shampoo-style Kronecker-factored
preconditioning with Adam in the preconditioner eigenbasis; in
contrast, PolarAdamW combines AdamW preconditioning with a polar
normalisation of the update matrix, so the two methods instantiate
the same broad ``adaptive + matrix-structured'' design pattern using
different matrix primitives.

\textbf{Concurrent Muon variants.}
Concurrent variants of Muon split into post-polar rescaling and
pre-polar preconditioning. Post-polar variants modify the magnitude
or normalisation after orthogonalisation:
PolarGrad~\citep{lau2025polargradclassmatrixgradientoptimizers} adds
nuclear-norm scaling, NorMuon~\citep{li2025normuonmakingmuonefficient}
adds row-wise second-moment normalisation, and
AdaMuon~\citep{si2025adamuonadaptivemuonoptimizer} applies element-wise
adaptive scaling to orthogonalised directions together with a
sign-stabilised orthogonal update. Pre-polar variants modify the
matrix fed to the polar map: FISMO~\citep{xu2026fismo} uses
Kronecker-factored Fisher structure, Mousse~\citep{zhang2026mousse}
uses Shampoo-style Kronecker-factored whitening,
MuonEq~\citep{chang2026muoneq} applies row/column equilibration before
Newton--Schulz, and Muon$^2$~\citep{liu2026muonsquared} applies
Adam-style second-moment preconditioning before orthogonalisation.
PolarAdamW is closest in design to Muon$^2$: both apply Adam-style
second-moment preconditioning before orthogonalisation, and the two
were developed independently. Muon$^2$ provides independent empirical
evidence on language-model pretraining that pre-polar Adam-style
preconditioning can improve over Muon. Here we use this design axis
for a different purpose: PolarAdamW serves as a controlled
AdamW-style intervention for the disentangling argument of
\cref{sec:gauge-equivariance}, preserving the polar spectral-control
mechanism while introducing a coordinatewise preconditioner (the
same one Muon$^2$ introduces) that is not equivariant under
orthogonal multiplicity-basis changes. These
works optimise the single-matrix update geometry; they do not analyse
covariance under Schur multiplicity-basis changes, which is the
mechanism isolated here.

\subsection{Equivariant networks and basis-equivariance}
\label{sec:related-equiv}

Group-equivariant
architectures~\citep{cohen2016groupequivariantconvolutionalnetworks,weiler2019generale2equivariantsteerablecnns,zhdanov2024cliffordsteerableconvolutionalneuralnetworks},
including tensor-product networks for atomistic systems (MACE,
NequIP)~\citep{batatia2022macehigherorderequivariant,Batzner_2022},
use equivariant parameterisations in which equivariant linear maps
reduce, after irreducible decomposition, to trainable
multiplicity-space blocks. Across this literature the optimiser is treated as a black box:
the constraint analysis is on the kernel parameterisation, and Adam or
AdamW is used by default. The closest formal neighbour is
\citet{nordenfors2025optimizationdynamicsequivariantaugmented}, who
prove that the gradient flow of a unitarily parameterised augmented
model preserves the equivariant subspace;
\cref{thm:basis-equiv} is the discrete-update analogue, showing
that the polar step is Schur gauge-equivariant under orthogonal change of the
multiplicity-space basis whereas coordinatewise AdamW preconditioning
depends on that basis. In this paper, equivariance enters on the
optimiser side: the SO(3) model supplies multiplicity-space matrices
on which basis-equivariance of the update is meaningful, and PolarAdamW
separates this property from spectral control.

\section{PolarAdamW: disentangling spectral control and Schur gauge-equivariance}
\label{sec:polaradamw}

PolarAdamW applies the same Newton--Schulz polar map as Muon, but to
the AdamW-preconditioned direction rather than to raw momentum. This
preserves the polar spectral-control mechanism while making the update
coordinate-preconditioned. The resulting loss of Schur
gauge-equivariance is analysed in
\cref{sec:gauge-equivariance}.

The three optimisers compared in this paper share the same auxiliary
AdamW step on all parameters outside the eligible matrix-structured
split (embeddings, LayerNorm parameters, biases, positional/class
tokens, patch embedding, and the classifier head). They differ only
in the update applied to the eligible two-dimensional weight matrices
inside transformer blocks.
For each two-dimensional weight matrix $W$, let
$g_t = \nabla_W \mathcal L_t$ be the mini-batch gradient at step $t$.
For Muon, $m_t = \mu\, m_{t-1} + g_t$ is heavy-ball momentum with
coefficient $\mu$. For AdamW and PolarAdamW, $m_t, v_t$ are the Adam
first and second moments with coefficients $\beta_1, \beta_2$, and
$\hat m_t = m_t/(1-\beta_1^t)$, $\hat v_t = v_t/(1-\beta_2^t)$ their
bias-corrected forms. The per-step direction $D_t$ on a 2D weight is:
\begin{align}
\textbf{Muon:}\quad      & D_t^{\mathrm{Muon}}       \;=\; \mathrm{NS}_5\!\left(g_t + \mu\, m_t\right); \label{eq:muon-update}\\[2pt]
\textbf{AdamW:}\quad     & D_t^{\mathrm{AdamW}}      \;=\; \frac{\hat m_t}{\sqrt{\hat v_t}+\varepsilon}; \label{eq:adamw-update}\\[2pt]
\textbf{PolarAdamW:}\quad& D_t^{\mathrm{PolarAdamW}} \;=\; \mathrm{NS}_5\!\left(\frac{\hat m_t}{\sqrt{\hat v_t}+\varepsilon}\right). \label{eq:polaradamw-update}
\end{align}
The matrix-weight update is $W \gets (1-\eta\lambda)W - \eta\, s\, D_t$,
with decoupled weight decay $\lambda$ on AdamW and PolarAdamW
matrix weights; for Muon's matrix step we set $\lambda = 0$ (Nesterov
momentum, matching the default Muon matrix-step setting in the
reference implementation~\citep{Jordan2024}).
The auxiliary AdamW step on parameters outside the eligible
matrix-structured split uses decoupled weight decay in all three arms;
\cref{sec:experiments-equivariant} reports a $\lambda = 0.05$
Muon audit on the SO(3) testbed.
For Muon and PolarAdamW we use the shape-dependent magnitude scale
$s = \sqrt{\max(n,m)/\min(n,m)}$ so that per-step Frobenius norms are
comparable across matrix shapes, the quintic Newton--Schulz
coefficients $(a,b,c) = (3.4445, -4.7750, 2.0315)$, and $k = 5$
iterations on $M / \|M\|_\Frob$ in bf16.

PolarAdamW forfeits Schur gauge-equivariance because the element-wise
denominator $\sqrt{\hat v_t}$ in equation~\cref{eq:polaradamw-update}
depends on the multiplicity-space basis: an orthogonal change of basis
rotates the entries of $\hat v_t$, and so changes the per-coordinate
rescaling. Muon's polar step has no such denominator and remains
Schur gauge-equivariant. The transformer and SO(3) experiments below
test the two sides empirically.

\section{Transformer experiments}
\label{sec:experiments-transformer}

The DeiT transformer blocks we study do not carry a nontrivial
Schur multiplicity gauge, so Schur gauge-equivariance cannot drive a
PolarAdamW gain over Muon here; such gains, if observed, would reflect
AdamW-style coordinatewise preconditioning applied before the polar
map. We test for such a gain on DeiT-Tiny trained on
random 100-class subsets of ImageNet-1k, with four paired seeds under
two recipes: a minimal-augmentation $100$-epoch recipe and a
$300$-epoch DeiT-style recipe (\cref{app:recipe-audit,app:recipe-data}). PolarAdamW outperforms Muon on all four seeds
under both recipes (\cref{tab:deit-minimal}).

We report \emph{mean-last-10 accuracy} (L10), the mean test accuracy
over the final $10$ epochs of each run. Each random 100-class subset
contains only about $5000$ test images, so single-epoch accuracy is
noisy and L10 is a more stable summary than final-epoch accuracy.

\begin{table}[ht]
\centering
\small
\begin{tabular}{lccc|ccc}
\toprule
& \multicolumn{3}{c|}{$100$ ep, minimal recipe}
& \multicolumn{3}{c}{$300$ ep, DeiT-style recipe} \\
seed & AdamW & Muon & PolarAdamW & AdamW & Muon & PolarAdamW \\
\midrule
$42$        & $63.08$ & $71.13$ & $\mathbf{73.11}$ & $79.80$ & $83.72$ & $\mathbf{85.03}$ \\
$7$         & $63.73$ & $72.17$ & $\mathbf{73.37}$ & $79.95$ & $83.82$ & $\mathbf{85.42}$ \\
$123$       & $64.56$ & $71.53$ & $\mathbf{73.75}$ & $79.42$ & $84.97$ & $\mathbf{86.10}$ \\
$2024$      & $64.49$ & $71.31$ & $\mathbf{73.65}$ & $80.75$ & $85.14$ & $\mathbf{86.56}$ \\
\midrule
mean        & $63.97$ & $71.53$ & $\mathbf{73.47}$ & $79.98$ & $84.41$ & $\mathbf{85.78}$ \\
$\Delta$(P$-$M)
            & ---     & ---     & $\mathbf{+1.93}$ & ---     & ---     & $\mathbf{+1.37}$ \\
\bottomrule
\end{tabular}
\caption{L10 test accuracy on DeiT-Tiny under two recipes (random
$100$-class ImageNet-1k subsets, four paired seeds). PolarAdamW
outperforms Muon by $+1.93$ pp at $100$ epochs and $+1.37$ pp at
$300$ epochs ($4/4$ seeds positive in both regimes). Both matrix-polar
optimisers outperform AdamW by $7$--$10$ pp.}
\label{tab:deit-minimal}
\end{table}

\paragraph{Robustness.}
Two concerns deserve checking. First, the random 100-class subset
protocol means absolute accuracies are subset-conditional, not
directly comparable to ImageNet-100 benchmarks; paired within-subset
gaps are the primary measurement (\cref{sec:open}). Second,
the minimal augmentation recipe could understate AdamW's true
performance and, in principle, reorder the optimisers under a stronger
recipe. The matched-augmentation audit at the same
$100$-epoch / batch-$128$ budget
(\cref{app:recipe-audit}) gives AdamW L10 $= 66.27$
(lr $= 5\times10^{-4}$), Muon $79.07$, and PolarAdamW $80.67$
(both lr $= 5\times10^{-3}$) on seed $42$: the matrix-optimiser
advantage is preserved under the stronger recipe.
The matched-augmentation comparison is single-seed in all three arms
and uses the learning rate selected in the minimal-recipe sweep, so we
treat it as an audit rather than a separately multi-seed-validated
result; the per-seed AdamW std on the minimal-recipe four-seed table
($0.7$ pp) is much smaller than the matrix-optimiser gap ($>12$ pp),
so the ordering is robust to plausible AdamW seed variance.

\paragraph{Post-hoc trajectory interpretation.}
The 300-epoch audit (\cref{app:recipe-audit-300ep}, trajectories
in \cref{fig:audit-300ep-trajectory}) shows an asymmetric
trajectory: Muon plateaus earlier, whereas PolarAdamW improves late
in the cosine schedule. We hypothesise that PolarAdamW composes AdamW-style second-moment
normalisation with the polar step. Early in training, changing second-moment estimates can
rotate the matrix fed to the polar map; late in training, once these
estimates stabilise and the learning rate is small, the
AdamW-preconditioned direction may become a more reliable local
rescaling. This is consistent with the late-cosine improvement
observed in our audit, but it is post hoc and is not a prediction of
\cref{thm:basis-equiv}.

\begin{figure}[!htbp]
\centering
\includegraphics[width=0.8\linewidth]{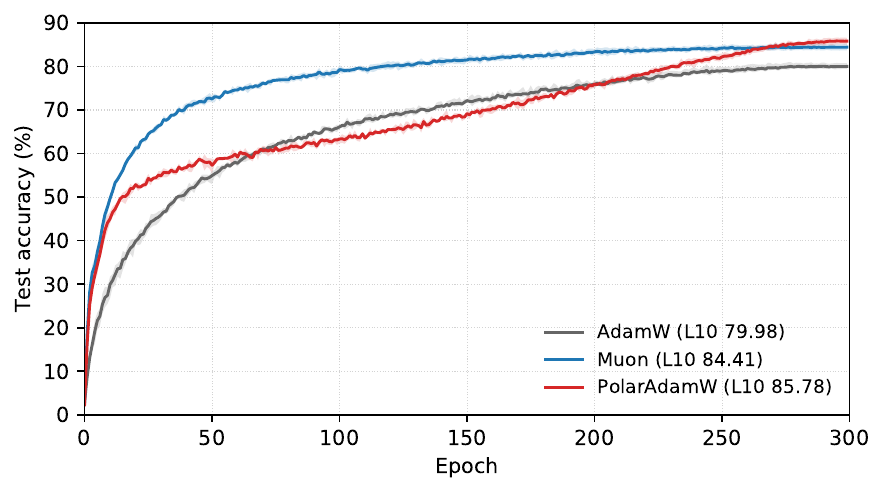}
\caption{Mean test-accuracy trajectories across the four seeds
$\{42, 7, 123, 2024\}$ of the $300$-epoch DeiT-style audit
(\cref{app:recipe-audit-300ep}); shaded bands span min/max across
seeds (narrow throughout, consistent with the per-seed table).
PolarAdamW overtakes Muon during the late cosine schedule, consistent
with the post-hoc regime-asymmetry interpretation above; both
matrix-polar arms remain well above AdamW throughout. L10
(ep $290$--$299$) values inset.}
\label{fig:audit-300ep-trajectory}
\end{figure}

\section{Schur multiplicity setting}
\label{sec:schur-setting}

The Schur isotypic decomposition exposes the natural parameter space
for equivariant networks; we use it to formalise the basis-equivariance
theorem of \cref{sec:gauge-equivariance} and to interpret the
SO(3) experiments of \cref{sec:experiments-equivariant}.

\subsection{General multiplicity decomposition}
\label{sec:setting-general}

Let $G$ be a compact group acting orthogonally on finite-dimensional inner-product
spaces $V$ and $V'$. Let $\{V_\lambda\}_\lambda$ index the irreducible
$G$-representations, and let $\rho_\lambda : G \to O(V_\lambda)$ denote the
action. Both spaces decompose as
\begin{equation}
V \cong \bigoplus_\lambda M_\lambda \otimes V_\lambda,
\qquad
V' \cong \bigoplus_\lambda N_\lambda \otimes V_\lambda,
\label{eq:isotypic}
\end{equation}
where $M_\lambda := \R^{m_\lambda}$ and $N_\lambda := \R^{n_\lambda}$
are the \emph{multiplicity spaces} of $\lambda$ in $V$ and $V'$,
respectively. We restrict throughout to real
irreducible types satisfying $\mathrm{End}_G(V_\lambda) = \R$, which
covers the SO(3) irreps used in our experiments. Every $V$ vector
splits as a sum of pieces labelled by an irrep $\lambda$; the
multiplicity $m_\lambda$ counts how many copies of $V_\lambda$
appear, and the irrep part $V_\lambda$ is where the group action
lives. $G$ acts trivially on $M_\lambda$, so any change of basis on
$M_\lambda$ leaves the $G$-action invariant; restricting to
orthogonal basis changes additionally preserves the inner product on
$M_\lambda$, fixing the Frobenius norm relevant to the polar map
below. We refer to this residual orthogonal multiplicity-basis freedom as
the \emph{Schur gauge}; it is central in
\cref{sec:gauge-equivariance}.

\begin{lemma}[Schur reduction of equivariant linear maps]
\label{lem:schur}
With the isotypic decompositions of \cref{eq:isotypic},
\begin{equation}
\Hom_G(V, V') \;\cong\; \bigoplus_\lambda \Hom(M_\lambda, N_\lambda) \otimes I_{V_\lambda}.
\label{eq:schur-reduction}
\end{equation}
Equivalently, an equivariant linear map $W : V \to V'$ is determined by a
tuple of \emph{multiplicity matrices} $B_\lambda \in \R^{n_\lambda \times m_\lambda}$,
acting on the $\lambda$-isotypic block as $B_\lambda \otimes I_{V_\lambda}$.
\end{lemma}

\noindent After choosing isotypic bases, the trainable degrees of
freedom in an equivariant linear layer are precisely the
multiplicity-space matrices
$B_\lambda : M_\lambda \to N_\lambda$. The SO(3) testbed below gives
a concrete instance of this reduction.

\subsection{The SO(3) point cloud setting}
\label{sec:setting-so3}

Our equivariant testbed is an SO(3)-equivariant network for
3D~point-cloud regression. Each point carries scalar (type-0) and
vector (type-1) channels, each with $h_c$ copies (the hidden-channel
count, $h_c = m_0 = m_1$ in the multiplicity notation above), and each
layer contains three Schur-structured trainable matrices: a
scalar-channel mixer, a vector-channel mixer (which acts as
$B \otimes I_3$ on the vector block), and a Clebsch--Gordan (CG) merge that takes
the $V_0$ component of $V_1 \otimes V_1$ (vector-channel dot products)
back into the scalar block; the $V_1$ and $V_2$ components of
$V_1 \otimes V_1 = V_0 \oplus V_1 \oplus V_2$ are not used. By the
Schur reduction \cref{eq:schur-reduction}, the linear mixers $W_{00},
W_{11}$ in each layer are stored directly as multiplicity matrices
($m_0 \times m_0$ and $m_1 \times m_1$). The CG merge layer
(\texttt{cg\_proj}) acts on scalar features together with
rotation-invariant dot products of vector channels; its weights are
equivariant by construction, but are not the same Schur multiplicity
block as $W_{11}$. The layer also contains auxiliary two-dimensional
message-MLP weights that are equivariant by construction but are not
the Schur blocks used in the basis-equivariance argument. The
matrix-structured optimisers act on the eligible two-dimensional
trainable arrays selected by the parameter split, including these
auxiliary weights. We refer to \cref{app:so3-arch} for the
full architectural walkthrough (parameter shapes, $B \otimes I_3$
implementation via \texttt{einsum}, and the $9 m_1^2 \to m_1^2$
reduction in parameter count under SO(3)-equivariance).

This contrasts with the DeiT setting of
\cref{sec:experiments-transformer}, where the relevant
transformer weights are not symmetry-imposed Schur multiplicity blocks.
The SO(3) model therefore supplies the setting in which
multiplicity-basis freedom is architecturally meaningful.

\medskip
\noindent\textbf{Notation.} For $M \in \R^{n \times m}$ with singular
value decomposition (SVD) $M = U \Sigma V^\top$, the polar factor is $\polar(M) := U V^\top$. We
denote by $\mathrm{NS}_k(M)$ the order-$k$ Newton--Schulz iteration used in
Muon, an odd polynomial in $M / \|M\|_\Frob$ that approximates
$\polar(M)$.

\section{Polar maps and Schur gauge-equivariance}
\label{sec:gauge-equivariance}

\label{sec:block-locality}

\cref{thm:basis-equiv} is the structural axis of the paper: it
identifies orthogonal multiplicity-basis covariance as a property
present in Muon's polar step and absent from AdamW's coordinatewise
step. \cref{lem:block-polar} sets up the block decomposition;
\cref{prop:adamw-not-equiv} establishes the negative case.
Detailed proofs are given in \cref{app:proofs}; below we
state each result and sketch the structural reason it holds.

\begin{lemma}[Polar respects direct sums]
\label{lem:block-polar}
Let $G = \bigoplus_\lambda G_\lambda \otimes I_{V_\lambda}$ be an
equivariant gradient (or momentum) tensor, with $G_\lambda \in \R^{n_\lambda \times m_\lambda}$.
Define the block-polar operator
\[
\polar_{\mathrm{block}}(G) \;:=\; \bigoplus_\lambda \polar(G_\lambda) \otimes I_{V_\lambda}.
\]
Then $\polar_{\mathrm{block}}(G)$ is again an equivariant tensor and,
under the same canonical polar convention on rank-deficient blocks,
coincides with the ordinary polar $\polar(G)$ taken on the full matrix
$G$ written in any basis adapted to the isotypic decomposition. The
same identity holds with $\mathrm{NS}_k$ in place of $\polar$ in exact
arithmetic with a common normalisation.
\end{lemma}

\paragraph{Intuition.}
The polar map is defined from the SVD and respects orthogonal block
decompositions, rather than the particular slicing of the matrix into
coordinates. When the matrix is
block-diagonal across irreducible representation types $\lambda$,
with each block of the form $G_\lambda \otimes I_{V_\lambda}$, the
$\otimes I_{V_\lambda}$ factor is invisible to the SVD up to
multiplicity, so polar applied to the whole equals polar applied to
each multiplicity-space matrix $G_\lambda$ separately. \emph{For the
exact polar map, applying the polar component to the full equivariant
gradient or momentum tensor produces the same block-structured update
as applying it separately to each multiplicity block $G_\lambda$.}
The same structural identity holds for finite-step Newton--Schulz
under shared scalar normalisation; the bf16 implementation is treated
empirically in \cref{app:rotation-test}.

\begin{remark}
For the exact polar map, this lemma shows that applying the polar map
to the full equivariant gradient or momentum tensor agrees with
applying it block-wise to each multiplicity-space block $G_\lambda$;
for finite-step Newton--Schulz the same holds under shared scalar
normalisation, and the bf16 deviation is quantified in
\cref{app:rotation-test}.
\end{remark}

The choice of basis on each multiplicity space $M_\lambda$ is a gauge
freedom: physically equivalent parameterisations of the equivariant network
differ by orthogonal change of basis on input and output multiplicity
spaces. The natural question is whether the optimiser's update direction
respects that gauge.

\begin{theorem}[Muon is multiplicity-basis equivariant]
\label{thm:basis-equiv}
For any $G_\lambda \in \R^{n_\lambda \times m_\lambda}$ and any orthogonal
$P_\lambda \in O(n_\lambda)$, $Q_\lambda \in O(m_\lambda)$,
\begin{equation}
\polar(P_\lambda\, G_\lambda\, Q_\lambda^\top)
\;=\;
P_\lambda \cdot \polar(G_\lambda) \cdot Q_\lambda^\top,
\label{eq:basis-equiv-polar}
\end{equation}
and the same identity holds with $\mathrm{NS}_k$ in place of $\polar$.
\end{theorem}

\paragraph{Intuition.}
The polar map sends $G$ to its SVD polar factor $U V^\top$, the
closest orthogonal factor in Frobenius norm. This depends on the
singular subspaces of $G$, not on the coordinate chart used to
express its entries. If
we rotate the input basis (apply $Q$) and the output basis (apply
$P$), the polar update transforms by the same rotations: the new
SVD is $(P U)\,\Sigma\,(Q V)^\top$, and the polar factor is
$(P U)(Q V)^\top = P\,(U V^\top)\,Q^\top$. The Newton--Schulz
iteration inherits this property because each step is built from
$X$ and $X^\top X$, both of which transform covariantly under
orthogonal conjugation.

\begin{proposition}[AdamW is Schur gauge-dependent]
\label{prop:adamw-not-equiv}
Let $\varepsilon > 0$, and let
\[
\bigl(\rho_\varepsilon(M)\bigr)_{ij}
\;:=\; \frac{M_{ij}}{|M_{ij}| + \varepsilon}
\]
denote the element-wise sign-like AdamW preconditioned direction in
the limiting case $\sqrt{\hat v_{ij}} = |M_{ij}|$. For orthogonal
changes of basis $P \in O(n)$, $Q \in O(m)$, the identity
\[
\rho_\varepsilon(P M Q^\top) \;=\; P\,\rho_\varepsilon(M)\,Q^\top
\]
does not hold in general.
\end{proposition}

\paragraph{Intuition.}
$\rho_\varepsilon$ depends on the magnitudes of the entries of $M$,
not on the linear map that $M$ represents. A change of basis that
mixes coordinate directions redistributes magnitudes across entries;
a single nonzero entry can become four equal-magnitude entries with
different signs, and the elementwise rule then produces an output
that no orthogonal post-correction can recover.
\cref{app:proofs-prop} gives an explicit two-dimensional
instance for any $\varepsilon > 0$.

\begin{remark}
\emph{Muon is intrinsic to the multiplicity-space linear map
$G_\lambda : M_\lambda \to N_\lambda$ as an operator. AdamW is
intrinsic only to the coordinate chart in which $G_\lambda$'s entries
are stored.} On Schur multiplicity matrices, where the basis is
gauge, Muon respects the gauge and AdamW does not.
\cref{prop:adamw-not-equiv} isolates the adaptive
coordinatewise-preconditioning step of AdamW; the decoupled
weight-decay term is scalar in $G$ and is not the source of the
gauge dependence.
\end{remark}

\begin{remark}
\cref{lem:block-polar} also holds for ordinary SGD and
scalar-momentum SGD (linear maps respect direct sums trivially) and
for AdamW (coordinatewise maps respect direct sums of
coordinate-aligned blocks). \cref{thm:basis-equiv}'s
covariance identity is also satisfied by SGD and scalar-momentum SGD,
since these are linear in $G$. The distinguishing combination for
Muon is therefore basis-equivariance \emph{plus} a per-block
spectral-norm control on the update: SGD/SGD-momentum lack the spectral
control; AdamW lacks both.
\end{remark}

\paragraph{Caveat.}
\cref{thm:basis-equiv} is a one-step covariance statement about
the update direction; convergence-rate consequences and quantitative
loss-of-progress bounds for AdamW under Schur gauge-dependence remain
open (see \cref{sec:open}). The SO(3) Muon--PolarAdamW comparison is
consistent with two non-exclusive readings: (a) architecture- and
scale-specific optimisation effects beyond Schur gauge-equivariance,
or (b) those effects plus a Schur gauge-equivariance bonus predicted
by \cref{thm:basis-equiv}; PolarAdamW provides a partial
separation by preserving polar spectral control while breaking Schur
gauge-equivariance. A clean separation requires an ablation that
breaks Schur gauge-equivariance while keeping Muon's exact momentum
dynamics; this remains open.

\section{Equivariant experiments: a regime reversal}
\label{sec:polar-adamw-double}
\label{sec:experiments-equivariant}

\cref{sec:gauge-equivariance} isolates a structural difference
between Muon and AdamW-preconditioned polar updates: Muon is Schur
gauge-equivariant, whereas PolarAdamW is not. If this difference is
an active ingredient in the SO(3) regime, we should expect a
\emph{double dissociation} when comparing Muon against PolarAdamW.

\paragraph{Numerical verification.} Before measuring this prediction
empirically, we verify \cref{thm:basis-equiv}'s covariance
identity at a representative multiplicity-matrix shape (hidden-channel
count $h_c = 8$); the full table across DeiT and SO(3) shapes is in
\cref{app:rotation-test}. We define
\[
\Delta(\phi; G, P, Q) \;:=\;
\frac{\|\phi(PGQ^\top) - P\phi(G)Q^\top\|_\Frob}{\|\phi(G)\|_\Frob}
\]
to measure how much an update map $\phi$ deviates from covariance
under the conjugation $G \mapsto PGQ^\top$. Over a sample of random
$(G, P, Q)$, both exact polar and full-precision Newton--Schulz give
$\Delta$ at numerical precision ($\sim 10^{-7}$); production bf16
Newton--Schulz (as Muon actually runs) is equivariant up to
finite-precision noise ($\Delta \sim 3$--$5\%$, consistent with bf16
mantissa width); AdamW's element-wise step gives $\Delta = \Theta(1)$,
confirming \cref{prop:adamw-not-equiv}. We refer to
\cref{app:rotation-test} for the full table, methodology, and
code reference. Thus the comparison below is made at shapes where the
structural distinction is numerically visible.

\paragraph{Equivariant architecture.}
We use a controlled SO(3)-equivariant point-cloud regression testbed:
an explicit Schur-block architecture (not a full E(3)-equivariant
network like MACE or NequIP), with targets
$\max \lambda(X^\top X / N)$ that are exactly SO(3)-invariant,
hidden-channel count $h_c \in \{16, 32, 64, 128\}$, and $100$ epochs
with $10$-epoch linear warmup, cosine decay, and early stopping at
patience $30$; architecture in \cref{app:so3-arch} and
training recipe in \cref{app:recipe-data}. The cross-axis 100-seed audit is
reported in \cref{tab:so3-mse}.

\paragraph{Weight-decay convention.}
All SO(3) arms run with weight decay $0$ throughout, by deliberate
design. The three optimisers handle weight decay differently (Muon's
matrix step has no wd parameter, following the reference
implementation; AdamW and PolarAdamW do), so applying $\text{wd}=0.05$
would impose arm-specific regularization rather than identical
regularization, defeating the goal of isolating the optimiser update.
The DeiT experiments in \cref{sec:experiments-transformer} use the
standard $\text{wd}=0.05$ ImageNet recipe to match practitioner-realistic
conditions. As a 4-seed convention audit, we re-ran the SO(3) comparison
under the $\text{wd}=0.05$ recipe: at $h_c = 16$, the Polar--Muon
paired $\Delta$ changes from $+0.00158$ (wd=0) to $+0.00174$
(wd=0.05), with all four seeds positive in both conditions; at
$h_c = 32$, AdamW at $\text{wd}=0.05$ gives mean test MSE
$0.00923$ versus the $\text{wd}=0$ baseline at $0.00918$. The SO(3) ordering and cross-axis
pattern are supported under the wd convention at the audited cells.

On this testbed, the PolarAdamW advantage observed on DeiT-Tiny
reverses, and the Muon edge grows with capacity:

\begin{table}[ht]
\centering
\begin{tabular}{ccccccc}
\toprule
$h_c$ & AdamW & Muon & PolarAdamW & $\Delta$(P$-$M) & paired $t$ & $n$ \\
\midrule
$16$  & $0.0099$ & $\mathbf{0.0067}$ & $0.0074$ & $+0.00070$ & $4.97$ & $100$ \\
$32$  & $0.0102$ & $\mathbf{0.0064}$ & $0.0074$ & $+0.00103$ & $7.63$ & $100$ \\
$64$  & $0.0103$ & $\mathbf{0.0055}$ & $0.0067$ & $+0.00120$ & $6.96$ & $100$ \\
$128$ & $0.0104$ & $\mathbf{0.0049}$ & $0.0063$ & $+0.00139$ & $8.62$ & $100$ \\
\bottomrule
\end{tabular}
\caption{Test MSE on SO(3)-equivariant point-cloud regression across
hidden-channel widths $h_c \in \{16, 32, 64, 128\}$, lower is better.
Entries are mean test MSE over $n$ paired seeds (column $n$). Muon
outperforms PolarAdamW at every width with $p < 10^{-5}$ (paired
$t$-test on common seeds), and the gap grows monotonically with
capacity. Both matrix-polar optimisers outperform AdamW, whose
performance does not improve with capacity. Recipe is fixed across
all cells; this is not a tuned benchmark.}
\label{tab:so3-mse}
\end{table}

\begin{remark}
The Muon edge over PolarAdamW is statistically significant at every
audited capacity ($p < 10^{-5}$, paired $t \in [4.97, 8.62]$) and
grows monotonically with $h_c$ (\cref{tab:so3-mse}). This is
consistent with the structural distinction isolated in
\cref{thm:basis-equiv}: Muon is Schur gauge-equivariant on
multiplicity matrices, whereas PolarAdamW introduces coordinatewise
AdamW preconditioning before the polar step, and the widening gap is
consistent with gauge-equivariance becoming more load-bearing as the
multiplicity-space dimension grows. We treat this as a structural consistency
claim rather than a causal proof; a Schur gauge-violation diagnostic
at training time (applying random orthogonal basis changes within
multiplicity spaces and comparing update norms) would yield a
stronger mechanism statement, and we leave this to future work.
\end{remark}

As a small robustness check, a four-seed learning-rate sweep at
$h_c = 32$ ($\eta \in \{3{\times}10^{-4}, 10^{-3}, 3{\times}10^{-3}\}$)
gives $\Delta(\mathrm{P} - \mathrm{M}) = +0.00111, +0.00084, +0.00014$
across the three rates; the 100-seed capacity audit
(\cref{tab:so3-mse}) is the primary evidence for the SO(3)
optimiser ordering.

This regime reversal is the primary empirical signature consistent
with the spectral-control / gauge-equivariance interpretation.

\section{Discussion and limitations}
\label{sec:open}

\paragraph{Contributions.}
The main contribution is a mechanism separation rather than a claim
that PolarAdamW is a standalone optimiser breakthrough. Analytically,
the paper contributes: (i) a Schur-multiplicity framework for asking
whether optimiser updates respect multiplicity-basis gauge freedom;
(ii) \cref{lem:block-polar}, showing that polar updates respect
the relevant isotypic block structure; (iii)
\cref{thm:basis-equiv}, proving Schur gauge-equivariance of
the polar/Muon step; and (iv) \cref{prop:adamw-not-equiv},
showing that AdamW-style coordinatewise preconditioning is
basis-dependent. Empirically, the SO(3) testbed supplies a regime in
which this distinction is structurally meaningful, while the
DeiT-Tiny experiments supply a complementary regime in which the
gauge is trivial.

PolarAdamW separates two effects that are coupled in Muon's polar step:
matrix-level spectral control and basis-equivariance under orthogonal
changes of multiplicity-space basis. Across architectures, the optimiser
ordering reverses: PolarAdamW gains over Muon on transformers, where
multiplicity-basis freedom is trivial; Muon retains a capacity-growing
advantage on SO(3)-equivariant networks, where this freedom is
non-trivial. This is consistent with the spectral-control versus
gauge-equivariance interpretation.

Both empirical sides have known scope. The transformer experiments use
random 100-class subsets of ImageNet-1k for paired multi-seed
comparison; a full-scale ImageNet-1k or language-modelling study,
beyond the compute budget available for this study, is the natural
next step. The SO(3) testbed similarly serves as a mechanism check
rather than a complete equivariant-learning benchmark; it uses one
tensor-field point-cloud architecture with type-0 and type-1 channels
and Clebsch--Gordan contractions. The matrix-structured optimisers also
act on auxiliary two-dimensional message-MLP weights inside this
architecture (\cref{app:so3-arch}), so the observed Muon edge
is not attributable purely to Schur gauge-equivariance on the
multiplicity blocks. Extending the comparison to standard
E(3)-equivariant atomistic networks such as
NequIP~\citep{Batzner_2022} and
MACE~\citep{batatia2022macehigherorderequivariant}, equivariant
attention, and higher-order tensor-field networks would clarify how
broadly the Muon--PolarAdamW contrast extends, as would testing
group-equivariant and steerable
CNNs~\citep{cohen2016groupequivariantconvolutionalnetworks,weiler2019generale2equivariantsteerablecnns},
whose equivariant kernel structures admit multiplicity-block
decompositions.

\cref{thm:basis-equiv} is a structural covariance statement
about one update direction; it does not give a convergence theorem or
a quantitative loss-of-progress bound for basis-dependent
coordinatewise preconditioning. A natural theoretical step is to
bound the basis-dependence gap of AdamW-style updates as a function
of the distribution of matrix entries under orthogonal basis changes,
connecting the Schur-gauge distinction to per-step optimisation cost.
A separate direction combines polar-normalised updates with explicit
parameter-space constraints, such as row-normalised weight matrices,
which changes the geometry of the parameter space rather than only
the update direction. Both directions lie outside the Schur-gauge
setting studied here. Extending the analysis to non-orthogonal basis
changes, changes of representation type, or transformations that mix
isotypic blocks would require additional structure.

\bibliographystyle{ims}
\bibliography{refs}

\appendix

\section{DeiT-Tiny augmentation-recipe audits}
\label{app:recipe-audit}

This appendix reports the matched-augmentation audits at 100 epochs
(Phase S1) and 300 epochs (Phase S2) on DeiT-Tiny under a DeiT-style
strong-augmentation recipe~\citep{touvron2021trainingdataefficientimagetransformers}.
Phase S1 is single-seed at $s{=}42$; Phase S2 is four-seed
($s \in \{42, 7, 123, 2024\}$) for all three optimisers.

\paragraph{Phase S1 configuration.}
The Phase S1 audit is one specific point in hyperparameter space, run
at the same $100$-epoch / batch-$128$ budget as the primary optimiser
arms in \cref{sec:experiments-transformer}. Exact values:

\begin{center}
\small
\begin{tabular}{ll}
\toprule
optimiser & AdamW (\texttt{torch.optim.AdamW}, $\beta_1=0.9$, $\beta_2=0.999$, $\varepsilon=10^{-8}$) \\
learning rate & $5\times10^{-4}$ (constant ``base''; cosine schedule below) \\
weight decay & $0.05$ (decoupled) \\
batch size & $128$ \\
total epochs & $100$ \\
warmup & $5$ epochs, linear from $\eta_0 = 10^{-3} \cdot \mathrm{lr}$ to $\mathrm{lr}$ \\
LR schedule & linear warmup, then cosine decay to $0$ over remaining $95$ epochs \\
Mixup $\alpha$ & $0.8$ (timm \texttt{Mixup}, mode \texttt{batch}, prob $1.0$, switch-prob $0.5$) \\
CutMix $\alpha$ & $1.0$ (combined with Mixup via switch) \\
RandAugment & \texttt{rand-m9-mstd0.5-inc1} (timm \texttt{create\_transform}) \\
label smoothing & $\varepsilon = 0.1$ (via timm \texttt{SoftTargetCrossEntropy}) \\
DropPath & $0.1$ (\texttt{drop\_path\_rate} on \texttt{vit\_tiny\_patch16\_224}) \\
RandomErasing & $p = 0.25$ (timm transform default) \\
image size & $224 \times 224$, bicubic interpolation \\
test transform & \texttt{Resize(256)} $\to$ \texttt{CenterCrop(224)} \\
seed & $42$ (single-seed audit) \\
\bottomrule
\end{tabular}
\end{center}

These values match the timm~\citep{wightman2019timm} / facebookresearch
DeiT~\citep{touvron2021trainingdataefficientimagetransformers}
\texttt{main.py} defaults for the listed knobs; they differ from the
official DeiT training setup in budget and scale ($100$ vs $300$
epochs; batch $128$ vs $1024$; no repeated augmentation; no EMA; no
colour jitter). The audit therefore measures \emph{recipe contribution
at this specific budget}, not the full official DeiT recipe outcome.

\paragraph{Result.} The same augmentation stack was run on all three
optimisers at $s{=}42$, $100$ ep, b$=128$, with each optimiser at the
learning rate that won its single-seed lr-scan on the minimal recipe:
AdamW lr $=5\times10^{-4}$, Muon lr $=0.005$, PolarAdamW lr $=0.005$.
Resulting L10 (test accuracy mean over ep 90-99):

\begin{center}
\begin{tabular}{lccc}
\toprule
optimiser & minimal L10 & matched-aug L10 & recipe lift \\
\midrule
AdamW       & $63.08$ & $66.27$         & $+3.19$ pp \\
Muon        & $71.13$ & $79.07$         & $+7.94$ pp \\
PolarAdamW  & $73.11$ & $\mathbf{80.67}$ & $+7.56$ pp \\
\bottomrule
\end{tabular}
\end{center}

Two observations follow.

\paragraph{Recipe-effect contrast.} The DeiT-style augmentation stack
benefits matrix-structured polar optimisers much more than plain AdamW
in this setup: AdamW gains $+3.19$ pp from augmentation, while Muon and
PolarAdamW gain $+7.94$ pp and $+7.56$ pp respectively. The
Muon-vs-AdamW gap under matched augmentation is therefore $\sim 12.80$
pp, substantially wider than the minimal-recipe gap of $\sim 8$ pp; the
gap is not a baseline artefact.

\paragraph{Trajectory shape: delayed-gain pattern (single seed).} The
DeiT-style augmentation stack changes the trajectory rather than simply
raising the endpoint. For PolarAdamW, augmentation initially lowers validation
accuracy by more than $9$ percentage points at epoch $25$ (minimal
$64.88\%$ vs matched-aug $55.46\%$), reflecting the harder mixed-label
and RandAugment objective. The augmented run catches the minimal recipe
around epoch $65$ and then pulls away, reaching L10 $= 80.67\%$ versus
$73.11\%$ for the minimal recipe. Thus the recipe effect is a
delayed-gain phenomenon: augmentation acts as an early optimisation
burden but yields a substantially better late-training regime. The
same qualitative pattern appears for Muon, whose recipe lift is
comparable. The crossover and four-seed minimal-recipe trajectories
are shown in \cref{fig:minimal-traj}.

\begin{figure}[h]
\centering
\includegraphics[width=0.8\linewidth]{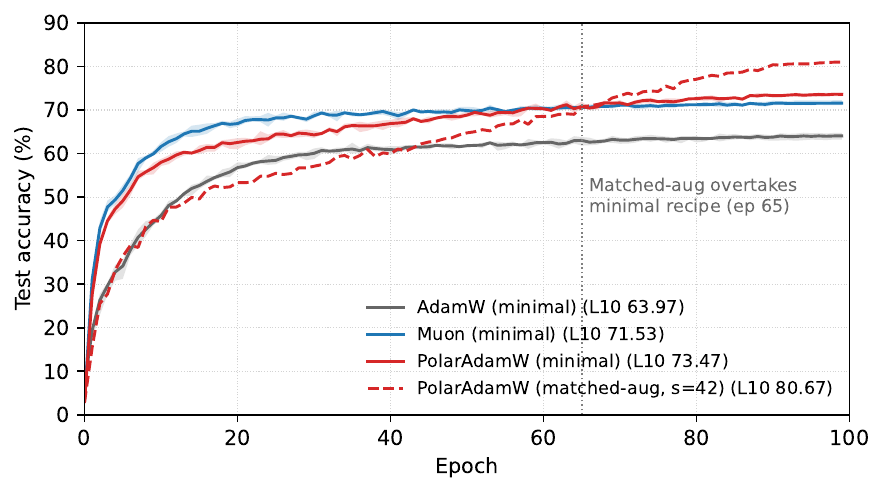}
\caption{DeiT-Tiny on four independently sampled $100$-class
ImageNet-1k subsets, $100$ ep. Solid curves are 4-seed means with
min/max envelopes for AdamW (gray), Muon (blue), and PolarAdamW
(red) under the minimal recipe; end-of-training L10 values match
the corresponding column of \cref{tab:deit-minimal}. The dashed
red curve is PolarAdamW under the DeiT-style augmentation stack
(matched-aug, single-seed at $s{=}42$); the vertical dotted line
marks ep $\sim 65$ where the matched-aug trajectory crosses its own
minimal-recipe envelope. Augmentation initially depresses accuracy
relative to the minimal recipe and then accelerates past it.}
\label{fig:minimal-traj}
\end{figure}

The matched-aug PolarAdamW result is single-seed at one lr; the
recipe-lift comparison ($+7.56$ vs $+7.94$ vs $+3.19$ pp) is the
observation here: in this audit, the matrix-structured polar
optimisers show a larger recipe lift than AdamW, with the lift
being delayed and only visible in the late-training trajectory.

\subsection*{Phase S2: $300$-epoch DeiT-style audit}
\label{app:recipe-audit-300ep}

To test whether the optimiser ordering at $100$ epochs survives a longer
DeiT-style training budget, the Phase~S1 recipe is extended to $300$
epochs (cosine schedule defined over $300$ epochs from the start),
matched across optimisers in everything except the per-arm learning
rate, with each optimiser at the lr that won its single-seed
minimal-recipe lr-scan: AdamW $5\times10^{-4}$, Muon $5\times10^{-3}$,
PolarAdamW $5\times10^{-3}$. All other hyperparameters identical to
Phase~S1. All three optimisers are run at four seeds
$s \in \{42, 7, 123, 2024\}$. Per-seed L10 (mean test accuracy over
ep $290$--$299$):

\begin{center}
\begin{tabular}{lcccc}
\toprule
seed & AdamW & Muon & PolarAdamW & $\Delta$ (Polar $-$ Muon) \\
\midrule
$42$    & $79.80$ & $83.72$ & $\mathbf{85.03}$ & $+1.31$ \\
$7$     & $79.95$ & $83.82$ & $\mathbf{85.42}$ & $+1.60$ \\
$123$   & $79.42$ & $84.97$ & $\mathbf{86.10}$ & $+1.13$ \\
$2024$  & $80.75$ & $85.14$ & $\mathbf{86.56}$ & $+1.42$ \\
\midrule
4-seed mean & $79.98$ & $84.41$ & $\mathbf{85.78}$ & $+1.37$ \\
\bottomrule
\end{tabular}
\end{center}

\noindent\emph{ImageNet experiments use random 100-class subsets of
ImageNet-1k, not the full dataset; absolute accuracies are
subset-conditional and are not benchmark claims. Paired within-subset
gaps are the primary measurement (see Limitations,
\cref{sec:open}).}

\paragraph{Reading.} At $300$ epochs under the DeiT-style recipe,
PolarAdamW remains ahead of Muon across all four seeds (mean
$\Delta = +1.37$ pp, all four positive). The mean $\Delta$ here is
smaller than the $+1.93$ pp at the $100$-epoch minimal recipe, but
recipe and budget both change between those two comparisons, so the
difference is not a clean budget-only effect. Every seed remains
positive at the longer budget. The narrowing is also compatible with
the within-run late-cosine catchup of PolarAdamW over Muon in
\cref{fig:audit-300ep-trajectory}: ``PolarAdamW gains on Muon
late within a run'' and ``the paired endpoint gap is larger at one
budget than another'' are claims on different axes. Going from the $100$-epoch DeiT-style recipe to the
$300$-epoch DeiT-style recipe at $s{=}42$ lifts AdamW by $+13.53$ pp, Muon
by $+4.65$ pp, and PolarAdamW by $+4.36$ pp. The matrix-vs-AdamW gap
is preserved at the long budget: across all four seeds, PolarAdamW
exceeds AdamW by $+5.80$ pp on average and Muon by $+4.43$ pp
(4-seed means $85.78$, $84.41$, $79.98$).

\section{\texttt{SO3EquivariantModel}: per-block walkthrough}
\label{app:so3-arch}

For SO(3), the irreducibles are $V_\ell$ of dimension $2\ell+1$:
$V_0 = \mathbb{R}$ (scalars; trivial action), $V_1 = \mathbb{R}^3$ (vectors;
rotated by $R$), $V_2$ (5-dim symmetric traceless 2-tensors), and so on.
\texttt{SO3EquivariantModel}
(\texttt{pointcloud\_\allowbreak so3/\allowbreak models.py}) uses
only types 0 and 1: each point in the cloud carries $m_0$ scalar
channels and $m_1$ vector channels, so its per-point feature space is
$V_0^{\oplus m_0} \oplus V_1^{\oplus m_1}$. We take
$m_0 = m_1 = \mathrm{hc}$ (the \texttt{hidden\_channels} flag, varied
over $\{16, 32, 64, 128\}$ in the cross-axis audit reported in
\cref{sec:polar-adamw-double}) for the generic equivariant layer; the
first layer takes $m_0^{\mathrm{in}} = m_1^{\mathrm{in}} = 1$ since
the per-point initial features are $\|x_i\|^2$ (a scalar) and $x_i$
itself (a vector). Each layer has three Schur-structured trainable
matrices:

\begin{enumerate}
\item[(i)] \textbf{Scalar mixer $W_{00}$.} An \texttt{nn.Linear($m_0$, $m_0$)},
weight $W_{00}.\mathrm{weight} \in \R^{m_0 \times m_0}$. By Schur's lemma,
$\Hom_G(V_0^{\oplus m_0}, V_0^{\oplus m_0}) \cong \R^{m_0 \times m_0}$
because $V_0$ is one-dimensional and the $\otimes\, \mathrm{id}_{V_0}$ is
invisible. \emph{The PyTorch tensor and the multiplicity matrix
$B_{\lambda=0}$ coincide.}

\item[(ii)] \textbf{Vector mixer $W_{11}$.} A na\"ive linear map
$\R^{m_1 \times 3} \to \R^{m_1 \times 3}$ would have $9 m_1^2$ entries, but
SO(3)-equivariance forbids most of them. By Schur,
$\Hom_G(V_1^{\oplus m_1}, V_1^{\oplus m_1}) \cong \R^{m_1 \times m_1}
\otimes \mathrm{id}_{V_1}$, so only $m_1^2$ free parameters survive: a
single matrix $B \in \R^{m_1 \times m_1}$ whose action on
$V_1^{\oplus m_1}$ is $B \otimes I_3$. The code stores exactly this:
\begin{verbatim}
# multiplicity matrix, not the full (3m, 3m) ambient form
self.W_11 = nn.Parameter(torch.empty(m1_out, m1_in))
v_out = torch.einsum('oi,bnid->bnod', self.W_11, vectors_in)
\end{verbatim}
The \texttt{einsum} \emph{is} $B \otimes I_3$: the same scalar entry
$B_{oi}$ multiplies all three components of vector channel~$i$. The
PyTorch tensor and the multiplicity matrix $B_{\lambda=1}$ coincide; the
$\otimes I_3$ is implicit in the indexing.

\item[(iii)] \textbf{Clebsch--Gordan merge $W_{\mathrm{cg}}$.} The tensor
product decomposes as $V_1 \otimes V_1 = V_0 \oplus V_1 \oplus V_2$. The
network keeps only the $V_0$ part: pairwise dot products among the $m_1$
output vectors give $n_{\mathrm{dot}} = m_1(m_1+1)/2$ extra scalar
features, which are concatenated with the existing $m_0$ scalars and
mixed back to $m_0$ scalars by an
\texttt{nn.Linear($m_0 + n_{\mathrm{dot}}$, $m_0$)}. This is again a
$V_0 \to V_0$ map, hence another multiplicity matrix
$B_{\lambda=0}^{(\mathrm{cg})} \in \R^{m_0 \times (m_0 + n_{\mathrm{dot}})}$.
\end{enumerate}

Beyond these per-layer Schur blocks, the model contains auxiliary
components: a distance-weighted message MLP and norm-gated vector
nonlinearities (within each layer), between-layer residual connections,
and the final invariant readout head. These
components are equivariant or invariant by construction, but they are
not the Schur blocks used in the basis-equivariance argument of
\cref{thm:basis-equiv}. In the optimiser experiments,
Muon/PolarAdamW are applied to all eligible two-dimensional arrays
selected by the parameter split, including auxiliary MLP weights; the
Schur-gauge discussion concerns the multiplicity-space arrays
$W_{00}$, $W_{11}$, and the scalar-sector $W_{\mathrm{cg}}$
(stored as \texttt{cg\_proj} in the code).

Thus $W_{00}$ and $W_{11}$ are Schur multiplicity matrices in the
strict sense. The $W_{\mathrm{cg}}$ weight is a scalar-sector
multiplicity matrix applied after the Clebsch--Gordan dot-product
features have been formed. These are the matrix-valued trainable
arrays in which the Schur-gauge structure appears explicitly; no
extraction from a larger ambient equivariant tensor is needed, and
the PyTorch tensors correspond exactly to the multiplicity matrices
$B_\lambda$.

\section{Detailed proofs}
\label{app:proofs}

We present the proofs of \cref{lem:block-polar},
\cref{thm:basis-equiv}, and
\cref{prop:adamw-not-equiv} in detail.

\subsection{Proof of \cref{lem:block-polar}}
\label{app:proofs-lem-block-polar}

\noindent\textbf{\cref{lem:block-polar}}
(\emph{Polar respects direct sums}).~\emph{Let $G = \bigoplus_\lambda
G_\lambda \otimes I_{V_\lambda}$ be an equivariant gradient (or
momentum) tensor, with
$G_\lambda \in \R^{n_\lambda \times m_\lambda}$. Define
$\polar_{\mathrm{block}}(G) := \bigoplus_\lambda \polar(G_\lambda)
\otimes I_{V_\lambda}$. Then $\polar_{\mathrm{block}}(G)$ is again an
equivariant tensor and, under the same canonical polar convention on
rank-deficient blocks, coincides with the ordinary polar $\polar(G)$
taken on the full matrix $G$ written in any basis adapted to the
isotypic decomposition. The same structural identity holds with
$\mathrm{NS}_k$ in place of $\polar$ in exact arithmetic with a common
normalisation; the bf16 implementation is treated empirically in
\cref{app:rotation-test}.}

\begin{proof}
Let each multiplicity-space block have a full SVD
$G_\lambda = U_\lambda \Sigma_\lambda V_\lambda^\top$ with
$U_\lambda \in O(n_\lambda)$, $V_\lambda \in O(m_\lambda)$, and
$\Sigma_\lambda \in \R_{\ge 0}^{n_\lambda \times m_\lambda}$
rectangular diagonal.
Tensoring with $I_{V_\lambda}$ distributes over matrix multiplication,
$(AB) \otimes (CD) = (A \otimes C)(B \otimes D)$ whenever the products
$AB$ and $CD$ are defined, so
\[
  G_\lambda \otimes I_{V_\lambda}
  = (U_\lambda \otimes I_{V_\lambda})\,
    (\Sigma_\lambda \otimes I_{V_\lambda})\,
    (V_\lambda \otimes I_{V_\lambda})^\top.
\]
The outer factors are themselves orthogonal: for any $U \in O(n)$,
$(U \otimes I)^\top (U \otimes I) = (U^\top U) \otimes I = I$. The
middle factor $\Sigma_\lambda \otimes I_{V_\lambda}$ is a rectangular
diagonal nonnegative matrix, with each singular value of $G_\lambda$
repeated $\dim V_\lambda$ times, so it is a valid rectangular
diagonal SVD factor.

Direct sums of orthogonal matrices are orthogonal, and direct sums of
rectangular diagonal nonnegative matrices are again rectangular
diagonal nonnegative, so assembling the multiplicity-space SVDs
across irreps gives a valid SVD of $G$:
\[
  G = \Bigl(\bigoplus_\lambda U_\lambda \otimes I_{V_\lambda}\Bigr)\,
      \Bigl(\bigoplus_\lambda \Sigma_\lambda \otimes I_{V_\lambda}\Bigr)\,
      \Bigl(\bigoplus_\lambda V_\lambda \otimes I_{V_\lambda}\Bigr)^\top.
\]
With the canonical polar convention fixed in the statement, reading
off this SVD gives
\[
  \polar(G)
  = \bigoplus_\lambda
    (U_\lambda \otimes I_{V_\lambda})(V_\lambda \otimes I_{V_\lambda})^\top
  = \bigoplus_\lambda (U_\lambda V_\lambda^\top) \otimes I_{V_\lambda}
  = \bigoplus_\lambda \polar(G_\lambda) \otimes I_{V_\lambda},
\]
which is the block-polar operator $\polar_{\mathrm{block}}(G)$. It is
$G$-equivariant because each summand has the intertwining form
$\polar(G_\lambda) \otimes I_{V_\lambda}$.

For the Newton--Schulz extension, one iteration acts as
$X \mapsto X \cdot p(X^\top X)$ for the polynomial
$p(t) = a + bt + ct^2$. Both $X^\top X$ and $X\,p(X^\top X)$ respect
block-diagonality and the $A \otimes I$ structure, since
$(A \otimes I)(B \otimes I) = AB \otimes I$. Iterating $k$ times
preserves the structure, giving
$\mathrm{NS}_k(G) = \bigoplus_\lambda \mathrm{NS}_k(G_\lambda) \otimes
I_{V_\lambda}$. Each summand again has the intertwining form
$B_\lambda \otimes I_{V_\lambda}$.
\end{proof}

\subsection{Proof of \cref{thm:basis-equiv}}
\label{app:proofs-thm}

\noindent\textbf{\cref{thm:basis-equiv}}
(\emph{Muon is multiplicity-basis equivariant}).~\emph{For any
$G_\lambda \in \R^{n_\lambda \times m_\lambda}$ and any orthogonal
$P_\lambda \in O(n_\lambda)$, $Q_\lambda \in O(m_\lambda)$,
\[
\polar(P_\lambda\, G_\lambda\, Q_\lambda^\top)
\;=\;
P_\lambda \cdot \polar(G_\lambda) \cdot Q_\lambda^\top,
\]
and the same identity holds with $\mathrm{NS}_k$ in place of $\polar$.}

\begin{proof}
We drop the subscript $\lambda$ throughout for clarity. Let
$G \in \R^{n \times m}$ have a full SVD $G = U \Sigma V^\top$, with
$U \in O(n)$, $V \in O(m)$, and
$\Sigma \in \R^{n \times m}$ rectangular diagonal nonnegative.
Substituting and using $V^\top Q^\top = (QV)^\top$,
\[
  P G Q^\top
  = P(U \Sigma V^\top)Q^\top
  = (PU)\,\Sigma\,(QV)^\top.
\]
The factors $PU$ and $QV$ are orthogonal: $(PU)^\top(PU) =
U^\top P^\top P U = U^\top U = I$, and similarly for $QV$. The middle
factor $\Sigma$ is unchanged, so $(PU,\,\Sigma,\,QV)$ is a valid SVD
of $PGQ^\top$. Under the same polar convention used throughout,
reading the polar factor off this SVD gives
\[
  \polar(P G Q^\top)
  = (PU)(QV)^\top
  = P\cdot(UV^\top)\cdot Q^\top
  = P\cdot\polar(G)\cdot Q^\top,
\]
which proves the polar identity.

For the Newton--Schulz extension, let one iteration be
$T(X) = X\,p(X^\top X)$ with $p(t) = a + bt + ct^2$. Then
$(PXQ^\top)^\top(PXQ^\top) = QX^\top X\,Q^\top$, and since $p$ is a
polynomial, $p(QAQ^\top) = Q\,p(A)\,Q^\top$. Hence
\begin{align*}
  T(PXQ^\top)
  &= (PXQ^\top)\,p(QX^\top X\,Q^\top) \\
  &= PXQ^\top\,Q\,p(X^\top X)\,Q^\top \\
  &= P\,X\,p(X^\top X)\,Q^\top
  \;=\; P\,T(X)\,Q^\top,
\end{align*}
where the third equality uses $Q^\top Q = I$. Iterating this
identity $k$ times gives $\mathrm{NS}_k(PGQ^\top) = P\,\mathrm{NS}_k(G)\,Q^\top$.
\end{proof}

\subsection{Proof of \cref{prop:adamw-not-equiv}}
\label{app:proofs-prop}

\noindent\textbf{\cref{prop:adamw-not-equiv}}
(\emph{AdamW is Schur gauge-dependent}).~\emph{Let
$\varepsilon > 0$, and let
$\bigl(\rho_\varepsilon(M)\bigr)_{ij} := M_{ij} / (|M_{ij}| +
\varepsilon)$ denote the element-wise sign-like AdamW preconditioned
direction in the limiting case $\sqrt{\hat v_{ij}} = |M_{ij}|$. For orthogonal
changes of basis $P \in O(n)$, $Q \in O(m)$, the identity
\[
\rho_\varepsilon(P M Q^\top) \;=\; P\,\rho_\varepsilon(M)\,Q^\top
\]
does not hold in general.
}

\begin{proof}
We show this by constructing a counterexample. Take $n = m = 2$,
$M = I_2$, $Q = I_2$, and
$P = R(\pi/4) = \tfrac{1}{\sqrt 2}\bigl(\begin{smallmatrix} 1 & -1 \\
1 & 1 \end{smallmatrix}\bigr)$ (rotation by $\pi/4$). For any
$\varepsilon > 0$,
\[
  \rho_\varepsilon(I_2) \;=\; \tfrac{1}{1 + \varepsilon}\,I_2,
\]
hence
$P\,\rho_\varepsilon(M)\,Q^\top = \tfrac{1}{1+\varepsilon}\,P$.
On the other hand, $PMQ^\top = P$, whose four entries all have
magnitude $1/\sqrt{2}$, so each entry of
$\rho_\varepsilon(P) = \rho_\varepsilon(PMQ^\top)$ takes the form
$P_{ij} / (1/\sqrt{2} + \varepsilon) =
\tfrac{\sqrt{2}}{1 + \sqrt{2}\,\varepsilon}\,P_{ij}$, giving
\[
  \rho_\varepsilon(PMQ^\top)
  \;=\; \tfrac{\sqrt{2}}{1 + \sqrt{2}\,\varepsilon}\,P.
\]
The two scalar factors satisfy $\tfrac{\sqrt{2}}{1+\sqrt{2}\varepsilon}
\neq \tfrac{1}{1+\varepsilon}$ for every $\varepsilon > 0$, since
equality would require $\sqrt{2}(1+\varepsilon) = 1 +
\sqrt{2}\,\varepsilon$, i.e., $\sqrt{2} = 1$. Therefore
$\rho_\varepsilon(PMQ^\top) \neq P\,\rho_\varepsilon(M)\,Q^\top$ at
this $(P, M, Q)$ for every $\varepsilon > 0$.
\end{proof}

\section{Numerical verification of Schur gauge-equivariance}
\label{app:rotation-test}

We numerically check the covariance identity of
\cref{thm:basis-equiv} across a range of matrix shapes: the
multiplicity-matrix shapes that appear in
\texttt{SO3EquivariantModel(hidden\_channels=8)}, three representative
DeiT-Tiny shapes, and additional small / medium / large square and
non-square shapes for breadth. For each shape, we sample $50$ random
$(G, P, Q)$ triples with $G \sim \mathcal{N}(0, I)$ and $P, Q$
Haar-orthogonal, then compute the conjugation deviation
\[
\Delta(\phi; G, P, Q) \;:=\;
\frac{\|\phi(P G Q^\top) - P\,\phi(G)\,Q^\top\|_\Frob}
     {\|\phi(G)\|_\Frob + 10^{-12}}.
\]
\cref{thm:basis-equiv} implies $\Delta \approx 0$ for
$\phi = \polar$ and for $\phi = \mathrm{NS}_k$;
\cref{prop:adamw-not-equiv} shows that no such covariance
identity holds for $\phi = \rho_\varepsilon$ at any $\varepsilon > 0$,
and the sampled deviations below are order one. The
$\Delta(\rho_0)$ column reports the near-sign limit, computed with
$\varepsilon = 10^{-30}$ only to avoid division by zero; the
deviation persists across all $\varepsilon > 0$.

\begin{table}[h!]
\centering
\small
\caption{Conjugation deviation $\Delta$ across Schur block, DeiT, and
synthetic matrix shapes; mean over 50 random $(G, P, Q)$ triples per
shape.}
\label{tab:basis-test}
\input{results/basis_equiv_table.tex}
\end{table}

Exact polar and fp64 Newton--Schulz are at numerical precision,
confirming \cref{thm:basis-equiv}. Production bf16
Newton--Schulz (as Muon actually runs) shows $\sim 3$--$5\%$
deviation, which reflects bf16 mantissa noise rather than failure of
the covariance identity; the fp64 implementation is at numerical
precision. The AdamW-style $\rho_0$ column has order-one deviation,
matching \cref{prop:adamw-not-equiv}.

\section{Recipe and data details}
\label{app:recipe-data}

\paragraph{Minimal recipe (AdamW, Muon, PolarAdamW arms).}
RandomResizedCrop$(224)$, RandomHorizontalFlip, and
ImageNet-statistics normalisation; \texttt{torch.optim.AdamW} with
$\beta_1{=}0.9$, $\beta_2{=}0.999$, $\varepsilon{=}10^{-8}$ at
$\mathrm{lr}_{\mathrm{aux}} = 5\times10^{-4}$ for parameters outside
the eligible matrix-structured split (embeddings, LayerNorm, biases,
positional/class tokens, patch embedding, and the classifier head); Muon and PolarAdamW at
$\mathrm{lr}_{\mathrm{polar}} = 5\times10^{-3}$ on 2D weights inside
\texttt{body.blocks.*}; batch $128$; $100$ epochs; 5-epoch linear
warmup from $\eta_0 = 10^{-3} \cdot \mathrm{lr}$ into cosine decay to
zero. No Mixup, CutMix, RandAugment, RandomErasing, label smoothing,
DropPath, EMA, or repeated augmentation. For Muon and PolarAdamW:
Newton--Schulz $5$-iteration with Jordan's quintic coefficients
$(3.4445, -4.7750, 2.0315)$ in bf16, momentum $\beta = 0.95$,
Adam-compat scaling $s = \sqrt{\max(n,m)/\min(n,m)}$ on the polar
update.

\emph{Weight decay.} The standalone AdamW arm uses decoupled
$\mathrm{wd}=0.05$ on all parameters. The Muon and PolarAdamW
hybrids apply $\mathrm{wd}=0.05$ to the auxiliary AdamW step on
non-matrix parameters. PolarAdamW additionally applies $0.05$ to its
matrix-parameter step (decoupled, AdamW-style). The Muon
matrix-parameter step itself uses $\mathrm{wd}=0$, the default of the
reference implementation~\citep{Jordan2024} (the class accepts a
\texttt{weight\_decay} argument). This is the only
inter-arm asymmetry in the weight-decay configuration.

\paragraph{Per-seed class subsets and overlap matrix.}
Each seed selects an independent $100$-class subset of ImageNet-1k via
a seeded random sample without replacement; the chosen indices are
stored at \texttt{datasets/\allowbreak imagenet-100/\allowbreak classes\_s\{seed\}.txt}. Pairwise class
overlap (number of shared classes out of $100$) across the four seeds
used in this paper:

\begin{center}
\small
\begin{tabular}{lcccc}
\toprule
& s=42 & s=7 & s=123 & s=2024 \\
\midrule
s=42   & 100 & 10  & 17  & 9   \\
s=7    & --- & 100 & 9   & 4   \\
s=123  & --- & --- & 100 & 6   \\
s=2024 & --- & --- & --- & 100 \\
\bottomrule
\end{tabular}
\end{center}

Mean overlap (off-diagonal) is $9.2$ out of $100$ ($9.2\%$); range is
$4$--$17$ shared classes. The per-seed split is therefore close to an
independent $100$-class task, and the multi-seed mean of paired gaps
estimates robustness across class samplings as well as across optimiser
randomness.

\paragraph{Same class subset across $100$-ep and $300$-ep runs.}
The $300$-epoch audit (\cref{app:recipe-audit-300ep}) uses the same
seed-$42$ random $100$-class ImageNet-1k subset as the corresponding
$100$-epoch experiment, so differences between these runs reflect recipe
and budget changes rather than a change in class subset.

\paragraph{Step-time measurement (DeiT-Tiny).}
Per-epoch wall-clock time was measured on a single NVIDIA A6000 with
batch size $128$ on DeiT-Tiny under the DeiT-style recipe, $300$-epoch
schedule (job 5133468). Mean per-epoch wall: AdamW $174$ s/ep, Muon
$183$ s/ep ($+4.7\%$ vs AdamW), PolarAdamW $185$ s/ep ($+6.0\%$ vs
AdamW; $+1.3\%$ vs Muon). The Newton--Schulz $5$-iteration step
(Jordan quintic coefficients, \texttt{bf16}) is the main additional
operation in PolarAdamW relative to AdamW; PolarAdamW's per-iteration
wall is essentially equal to Muon's, supporting the
``comparable to Muon'' claim in the abstract. We do not normalise by
FLOPs because the polar step's cost is batch-size invariant and its
relative overhead therefore depends on batch size; we report the
wall-clock number that actually obtains in the DeiT comparisons.

\paragraph{Step-time measurement (SO(3)-equivariant model).}
Per-run wall-clock time on the SO(3) testbed was measured across the
$100$-seed audit on the cluster's eight-short partition (4 CPUs, $4$
GB), $100$ epochs per run. Mean per-run wall (s, $100$ seeds per
cell):

\begin{center}
\begin{tabular}{lccc}
\toprule
$h_c$ & AdamW & Muon & PolarAdamW \\
\midrule
$16$  & $629$   & $657$ ($+4.6\%$)    & $590$ ($-6.1\%$ vs AdamW; $-10.2\%$ vs Muon) \\
$32$  & $1353$  & $1445$ ($+6.8\%$)   & $1392$ ($+2.8\%$ vs AdamW; $-3.7\%$ vs Muon) \\
$64$  & $4606$  & $4915$ ($+6.7\%$)   & $5151$ ($+11.8\%$ vs AdamW; $+4.8\%$ vs Muon) \\
$128$ & $17768$ & $23025$ ($+29.6\%$) & $21068$ ($+18.6\%$ vs AdamW; $-8.5\%$ vs Muon) \\
\bottomrule
\end{tabular}
\end{center}

Within-cell standard deviation is $\approx 12$--$21\%$ of the mean,
reflecting cluster contention on shared CPU nodes. At $h_c\in\{16,32,64\}$,
the cross-optimiser differences sit within or near the within-cell standard
deviation, so we do not infer a reliable optimiser-side timing ordering at
smaller capacities. At $h_c=128$, polar-step overhead is visible above the
noise floor: Muon runs $+29.6\%$ slower than AdamW and PolarAdamW $+18.6\%$,
with PolarAdamW faster than Muon in the paired timing comparison
($\Delta_{\mathrm{Polar-Muon}}=-1957$ s, $t=-3.36$). Because the Muon and
PolarAdamW implementations use the same Newton--Schulz routine, we treat this
last comparison as a wall-clock observation rather than evidence of a stable
algorithmic cost ordering. The \texttt{cg\_proj} layer
(\cref{app:so3-arch}, $O(h_c^4)$) dominates per-run compute at
smaller capacities; polar-step overhead becomes visible only at the largest
audited width.

\paragraph{Hyperparameter tuning budget.}
Each optimiser was tuned via a single-seed ($s{=}42$) lr-scan on the
minimal recipe at $100$ epochs, with the winning lr carried forward to
all multi-seed cells and to the DeiT-style / long-budget audits. The
scanned grids: AdamW $\mathrm{lr} \in \{1, 3, 5, 10\}\times 10^{-4}$
(winner $5\times 10^{-4}$); Muon $\mathrm{lr} \in \{0.002, 0.005,
0.01, 0.02\}$ (winner $0.005$); PolarAdamW scanned the same grid as
Muon (winner $0.005$). For DeiT, decoupled weight decay is $0.05$ on
AdamW parameters and on the PolarAdamW matrix step; Muon's matrix step
uses no weight decay (the default of the reference implementation,
which accepts a \texttt{weight\_decay} argument; SO(3) arms use weight
decay $0$ throughout, see \cref{sec:experiments-equivariant}). The remaining hyperparameters are
held fixed across optimisers and not tuned per-arm: $5$-epoch linear
warmup, cosine decay, $\beta_1{=}0.9$, $\beta_2{=}0.999$,
$\varepsilon{=}10^{-8}$. An additional AdamW lr-scaling audit at the
DeiT-style recipe (matched-aug, $100$ ep, single-seed at $s{=}42$, $\mathrm{lr}
\in \{6.25\times 10^{-5}, 1.77\times 10^{-4}\}$ as commonly recommended
for batch-size-scaled AdamW) yielded L10 $54.71$ and $63.63$
respectively, both well below the $\mathrm{lr}=5\times 10^{-4}$ audit
($66.27$). Within this tuning
budget, the AdamW baseline is not explained by an obviously poor
learning-rate choice; the $+1.93$ pp four-seed mean PolarAdamW
advantage remains after the AdamW lr audits.

\end{document}

%% file: results/basis_equiv_table.tex
\begin{tabular}{lcccc}
\toprule
Shape $(n,m)$ & $\Delta(\polar)$ & $\Delta(\mathrm{NS}_5^{\mathrm{fp64}})$ & $\Delta(\mathrm{NS}_5^{\mathrm{bf16}})$ & $\Delta(\rho_0)$ \\
\midrule
W\_00 / W\_11 layer 0 (8,1) & 4.0e-08 & 4.0e-08 & 3.2e-02 & 0.78 $\pm$ 0.18 \\
W\_00 / W\_11 layer k (8,8) & 1.7e-07 & 2.3e-07 & 5.4e-02 & 0.86 $\pm$ 0.06 \\
cg\_proj (8,44) & 1.4e-07 & 1.8e-07 & 4.6e-02 & 0.85 $\pm$ 0.03 \\
DeiT-Tiny attn.qkv (192,192) & 4.1e-07 & 5.2e-07 & 3.5e-02 & 0.85 $\pm$ 0.00 \\
DeiT-Tiny mlp.fc1 (768,192) & 3.2e-07 & 5.2e-07 & 3.7e-02 & 0.85 $\pm$ 0.00 \\
tiny square (4,4) & 1.4e-07 & 1.9e-07 & 5.2e-02 & 0.79 $\pm$ 0.13 \\
small square (16,16) & 2.1e-07 & 2.5e-07 & 4.5e-02 & 0.85 $\pm$ 0.03 \\
medium square (32,32) & 2.6e-07 & 3.1e-07 & 4.3e-02 & 0.85 $\pm$ 0.02 \\
medium-large square (128,128) & 3.7e-07 & 4.5e-07 & 3.6e-02 & 0.85 $\pm$ 0.00 \\
tall non-square (4,16) & 1.1e-07 & 1.7e-07 & 5.3e-02 & 0.85 $\pm$ 0.07 \\
wide non-square (16,4) & 1.1e-07 & 1.7e-07 & 5.0e-02 & 0.84 $\pm$ 0.07 \\
tall non-square (8,32) & 1.4e-07 & 2.0e-07 & 4.7e-02 & 0.85 $\pm$ 0.03 \\
wide non-square (32,8) & 1.5e-07 & 2.2e-07 & 4.8e-02 & 0.86 $\pm$ 0.03 \\
DeiT-Tiny mlp.fc2 (192,768) & 3.2e-07 & 5.2e-07 & 3.7e-02 & 0.85 $\pm$ 0.00 \\
large square (768,768) & 5.4e-07 & 5.9e-07 & 3.1e-02 & 0.85 $\pm$ 0.00 \\
\bottomrule
\end{tabular}